\documentclass[letterpaper]{article}
\usepackage{proceed2e}
\usepackage[margin=1in]{geometry}

\usepackage{color}
\usepackage{amsmath}
\usepackage{amssymb}
\usepackage{amsthm}
\usepackage{algorithm}
\usepackage{algorithmic}
\usepackage{booktabs}

\usepackage{graphicx}

\newtheorem{definition}{Definition}[section]
\newtheorem{lemma}{Lemma}[section]
\newtheorem{theorem}{Theorem}[section]
\newtheorem{proposition}{Proposition}[section]

    \newtheoremstyle{TheoremNum}
        {\topsep}{\topsep}              %%% space between body and thm
        {\itshape}                      %%% Thm body font
        {}                              %%% Indent amount (empty = no indent)
        {\bfseries}                     %%% Thm head font
        {.}                             %%% Punctuation after thm head
        { }                             %%% Space after thm head
        {\thmname{#1}\thmnote{ \bfseries #3}}%%% Thm head spec
    \theoremstyle{TheoremNum}
    \newtheorem{thmn}{Theorem}

\DeclareMathOperator*{\argmax}{argmax}

\usepackage{natbib}
\renewcommand\bibsection{}

\usepackage{times}

\title{Fast Amortized Inference and Learning in Log-linear Models with Randomly Perturbed Nearest Neighbor Search}

\author{} % LEAVE BLANK FOR ORIGINAL SUBMISSION.
\author{{\bf Stephen Mussmann\thanks{\hspace{0.2pc} Both authors contributed equally.}} , {\bf Daniel Levy$^*$}, {\bf Stefano Ermon}    \\
Department of Computer Science \\
Stanford University\\
Stanford, CA 94305 \\
\texttt{\{mussmann,danilevy,ermon\}@cs.stanford.edu} \\
}

\begin{document}

\maketitle

\begin{abstract}
Inference in log-linear models scales linearly in the size of output space in the worst-case. This is often a bottleneck in natural language processing and computer vision tasks when the output space is feasibly enumerable but very large. We propose a method to perform inference in log-linear models with sublinear amortized cost. Our idea hinges on using Gumbel random variable perturbations and a pre-computed Maximum Inner Product Search data structure to access the most-likely elements in sublinear amortized time. Our method yields provable runtime and accuracy guarantees. Further, we present empirical experiments on ImageNet and Word Embeddings showing significant speedups for sampling, inference, and learning in log-linear models.
\end{abstract}

\section{INTRODUCTION}

Log-linear models are widely used in machine learning and statistics. These models receive their name from the fact that the log unnormalized probabilities are linear in the parameters and the sufficient statistics. Since the probabilities are defined up to scaling, inference and learning require computing the normalization constant, also known as the partition function \citep{murphy2012machine}.

While defining unnormalized probabilities affords modeling flexibility, it comes at the price of computation time. For factorized models, there are many methods, such as Gibbs sampling and variational inference \citep{koller2009probabilistic}, to approximately perform inference. Here we are interested in the setting where the output space is not factorizable, and is large but enumerable (e.g., a few million elements). Such problems with large output spaces occur in many areas including computer vision and natural language processing (NLP) \citep{joulin2016learning, bengio2003neural,mikolov2013distributed}.
While inference is tractable (by brute force, in time linear in the size of the output space), it can be a major bottleneck in learning and even at test time in resource-constrained settings. 
Clearly, computation time cannot be saved for a \emph{single} inference query as it requires linear time to examine the input. However, as \citet{mussmann2016learning} establishes, computation time can be saved for a \emph{sequence} of related queries, e.g., sampling from log-linear models with the same sufficient statistics but different (changing) parameters. Such sequences of queries arise naturally in learning and at test time.

In this work, we employ Gumbel random variables to convert sampling into maximizing unnormalized log-probabilities perturbed by Gumbel noise, applied independently to each element in the output space \citep{hazan2013sampling,maddison2014sampling,kim2016exact}. Naively, sampling a Gumbel for each element requires linear runtime which yields no savings. However, we introduce a novel way to lazily instantiate the Gumbel random variables. In order to maximize the Gumbel-perturbed objective, we only examine a (small) subset of the most-likely states and a small number of Gumbel perturbations (the largest ones). This yields asymptotic runtime improvements with provable accuracy guarantees. To find the most likely states, which involves the maximization of the dot product between the parameters and the sufficient statistics, we are able to make use of the large literature on Maximum Inner Product Search \citep{shrivastava2014asymmetric,auvolat2015clustering,douze2016polysemous,ram2012maximum,koenigstein2012efficient}.

The contributions of this work are as follows.

\begin{itemize}
\item We present a method to perform sampling using access to the top $O(\sqrt{n})$ most likely states (where $n$ is the number of states), using the Gumbel max trick.
\item We present a method to estimate the partition function and expected values using access to the top $O(\sqrt{n})$ values and relying on uniform sampling.
\item We present a way to use Maximum Inner Product Search (MIPS) techniques to retrieve the approximate top $O(\sqrt{n})$ elements to provably achieve sublinear amortized query time.
\item We demonstrate applications of our method in computer vision and NLP where we achieve $5$--$10\times$ per-query speedups compared to the naive method.
\end{itemize}

\section{BACKGROUND}
\label{sec:background}

\subsection{LOG-LINEAR MODELS}

Log-linear models are widely used  in machine learning and artificial intelligence. Generally, any exponential family distribution can be written as a log-linear model. As examples, very common models like multinomial logistic regression and maximum entropy models \citep{koller2009probabilistic,murphy2012machine} are log-linear models. Additionally, the last layer of a neural network (softmax) is a log-linear model, e.g. sampling the next element of a sequence for recurrent neural networks.

In this work, we focus on discrete distributions over a set of states $\mathcal{X}$. For a log-linear model, the log unnormalized probabilities are linear in the parameters. More precisely, if the parameters are $\theta$ and the features (sufficient statistics) for an element $x \in \mathcal{X}$ are $\phi(x)$, then,
\begin{equation}
\Pr(x; \theta) \propto e^{\theta \cdot \phi(x)}
\end{equation}
Note that in order to define a distribution, we must normalize these probabilities by 
\begin{equation}
Z_\theta = \sum_{x \in \mathcal{X}} e^{\theta \cdot \phi(x)}
\end{equation}
which is known as the partition function. Unfortunately, computing the partition function $Z$ is expensive as it requires summing over all elements in $\mathcal{X}$. We can also learn a log-linear model by maximizing the likelihood of some training data where evaluating the gradient requires computing the expected value of the sufficient statistics.

\textbf{Assumption:} In our setting, $\mathcal{X}$ is large but feasibly enumerable, so naively computing the partition function is tractable but computationally expensive. 

As an example, in the experimental results section, $|\mathcal{X}| \approx 10^6$. As a negative example, Markov Random Fields can be written as a log-linear model but have an exponentially large $\mathcal{X}$ and thus are not amenable to our method.

\subsection{GUMBEL VARIABLE}
In the context of extremal statistics, \citet{gumbel1954statistical} defines the Gumbel distribution as
\begin{equation}
\Pr(G < x) = \exp(-\exp(-x))
\end{equation}
We can sample a Gumbel random variable using the following scheme,
\begin{equation}
U \sim \text{Uniform}(0,1)
\end{equation}
\begin{equation}
G = -\ln(-\ln(U))
\end{equation}
Our use of the Gumbel distribution is motivated by the so-called ``Gumbel Trick'' which involves adding Gumbel noise to the log unnormalized probabilities to turn sampling from a log-linear model into finding the maximizing element. 

\begin{proposition}[\citep{hazan2013sampling,maddison2014sampling}]
\label{prop:gumbel-max}
For Gumbel variables $G_x$ sampled i.i.d. for each data point $x$,
\begin{equation}
\argmax_x \theta \cdot \phi(x) + G_x \sim \mathrm{Categorical}(\{\frac{e^{\theta \cdot \phi(x)}}{Z_\theta} \}_x)
\end{equation}
\end{proposition}

\subsection{MAXIMUM INNER PRODUCT SEARCH}\label{subsec:mips}

A common computational task is retrieving the nearest neighbor to a query from a database of vectors. More specifically, we are given a database of vectors on which we can perform preprocessing and build a data structure, and then, we receive a sequence of queries $\{q_i\}_i$, and for each query, we use the data structure to compute the element in the database that is most similar to $q_i$. Note that the structure of this problem depends on the similarity measure between vectors.

If $n$ is the number of vectors, we can trivially create an $O(n)$ algorithm (per query): for every query $q$, iterate through the entire database and find the vector that is most similar to $q$. Remarkably, for Euclidean distance and cosine similarity, it is possible to achieve \emph{amortized sublinear} query runtime \citep{indyk1998approximate,charikar2002similarity}.

Because of applications in log-linear models, we will be interested in using the inner product as the similarity measure. This is known as the Maximum Inner Product Search (MIPS) task.

\begin{definition}[Maximum Inner Product Search]
Given a set of vectors $V = \{v_1, ..., v_n\}$ the MIPS task is to respond to a query vector $q$ with
\begin{equation}
\argmax_{v \in V} q \cdot v
\end{equation}
\end{definition}

One common class of techniques for solving MIPS are space-partitioning methods such as k-d trees \citep{bentley1975multidimensional}. \citet{ram2012maximum} and \citet{koenigstein2012efficient} introduce space-partitioning methods based on a branch and bound technique to solve the MIPS problem. Unfortunately, it has been observed that such tree-based methods suffer from the curse of dimensionality \citep{shrivastava2014asymmetric}.

Clustering is another approach for solving the MIPS task \citep{auvolat2015clustering,douze2016polysemous}. For this technique, the database vectors are clustered during the preprocessing step. Then, at query time, the algorithm searches the clusters near $q$ for the most similar vector.

Another common class of techniques for MIPS are based on Local Sensitive Hashing  \citep{shrivastava2014asymmetric, neyshabur2014symmetric}, a method introduced by \citet{indyk1998approximate}. LSH only requires a family of hash functions with collision probabilities that are monotonic in the similarity. LSH works by combining these basic hashes to form longer hashes, and then building a hash table for each longer hash. Then, at query time, LSH hashes the query, retrieves elements from the colliding hash buckets, and computes the maximum over such elements. More precisely, define $\mathrm{Sim}(x,y)$ as the similarity between $x$ and $y$ and an $S$-neighbor to a query $q$ as a point $x$ such that $\mathrm{Sim}(q,x) \geq S$.

\begin{theorem}
Given a set $V$ of size $n$ with a similarity measure and hash family $\mathcal{H}$ such that for scalars $S_1 > S_2$ and $p_1 > p_2$,
\begin{itemize}
\item For any $x,y \in V$ where $\mathrm{Sim}(x,y)\geq S_1$, $\Pr_{h \in \mathcal{H}}[h(x) = h(y)] \geq p_1$
\item For any $x,y \in V$ where $\mathrm{Sim}(x,y)\leq S_2$, $\Pr_{h \in \mathcal{H}}[h(x) = h(y)] \leq p_2$
\end{itemize}
one can construct a data structure which, given any query $q$, does the following with high probability: if there exists a $S_1$-neighbor of $q$ in $V$, it returns a $S_2$-neighbor of $q$ in $V$. Further, this can be done with $O(n^\rho \log n)$ query time and $O(n^{1 + \rho})$ space where $\rho = \frac{\log p_1}{\log p_2} < 1$.

\end{theorem}
\begin{proof}
See \citet{indyk1998approximate}
\end{proof}

This theorem states that if there is an $S_1$-close neighbor, the algorithm will find an $S_2$-close neighbor. Intuitively, this means that each LSH instance is ``tuned'' to a different similarity value. We can build a series of LSH instances ``tuned'' to different values so that we can find the largest element with high probability, no matter the similarity of the nearest neighbor. The theorem states that this can be done in sublinear time.

For the sublinear theoretical guarantees in this paper, we will rely on the reduction from MIPS to Maximum Cosines Similarity Search presented in \citet{neyshabur2014symmetric} which adds a single dimension to make all the database vectors have the same norm. For the cosine similarity search problem, we will rely on LSH techniques for cosine similarity search presented in \cite{charikar2002similarity} based on Signed Random Projections, a binary hash function based on the sign of the dot product with a random vector.

\section{METHOD}
\label{sec:method}

Suppose we have a log-linear model over a set of $n$ elements $\mathcal{X}$. We wish to perform sampling and inference in sublinear time. This cannot be done for a single value of the parameters $\theta$, but with preprocessing on $\mathcal{X}$, we can achieve sublinear amortized query time. 

Our method generally works for any distribution where
\begin{equation}
\Pr(i) = \frac{e^{y_i}}{\sum_j e^{y_j}}
\end{equation}
which encompasses all distributions with strictly positive probability mass. The requirement for our method is that we have access to the largest $O(\sqrt{n})$ values of $y_i$ in sublinear time. In particular, this method works for log-linear models where $y_i = \theta \cdot \phi(x_i)$ and we use Maximum Inner Product search techniques to access the top values. 

\subsection{SAMPLING}

Recall the Gumbel max technique from the background section. In particular, if we can compute the maximum element and value of $y_i + G_i$ for Gumbel variables $G_i$, the maximum element will be a sample from the model. We can construct a naive strategy  as follows: sample a Gumbel $G_i$ for each $y_i$ and iterate over all elements to find the maximum (perturbed) element. However, this algorithm's runtime is linear and provides no savings. 

Ideally, we would like to find a way to preprocess $\{y_i\}_{i=1}^n$ so that we can draw samples and perform inference quickly. \citet{mussmann2016learning} achieves this by performing preprocessing on fixed Gumbel samples and using MIPS techniques. This ``frozen'' Gumbel noise makes the samples very correlated; in fact, there are a small fixed number of possible samples for a given parameter value. We wish to find a way that allows us to sample fresh Gumbels for every sample, but only a sublinear number of them.

Intuitively, for an element to maximize $y_i + G_i$, either $y_i$ needs to be large or $G_i$ needs to be large, so we only need to examine indices where either $y_i$ or $G_i$ is large. We can find the largest $y_i$ by performing preprocessing (such as maximum inner product search) and we will find the largest $G_i$ by incorporating a lazy evaluation strategy for the Gumbel variables to only require an expected number of $O(\sqrt{n})$ samples.

First, we describe the method intuitively and with a figure before diving into the details. Let $S$ be the set of the largest $O(\sqrt{n})$ elements of $\{y_i\}_{i=1}^n$. First, we will sample Gumbel values for these $S$ largest elements. Note that the minimal $y_i$ in $S$ is an upper bound on the $y_i$ value of elements not in $S$. Further, for an element not in $S$ to have the overall maximal $y_i + G_i$, it must exceed the maximal $y_i + G_i$ for elements in $S$, which is quickly computable. Thus, we have a lower bound on what the value of a Gumbel must be to perturb a point not in $S$ to be the overall maximum. We can lazily sample large Gumbels that exceed this gap, which we will show there will not be too many in expectation. Then, we randomly assign these large Gumbels to the tail of the distribution and check if any of them exceed the maximal $y_i + G_i$ from the largest elements $S$. See Figure \ref{fig:lazy_sample}.

\begin{figure}
\centering
\includegraphics[width=\columnwidth]{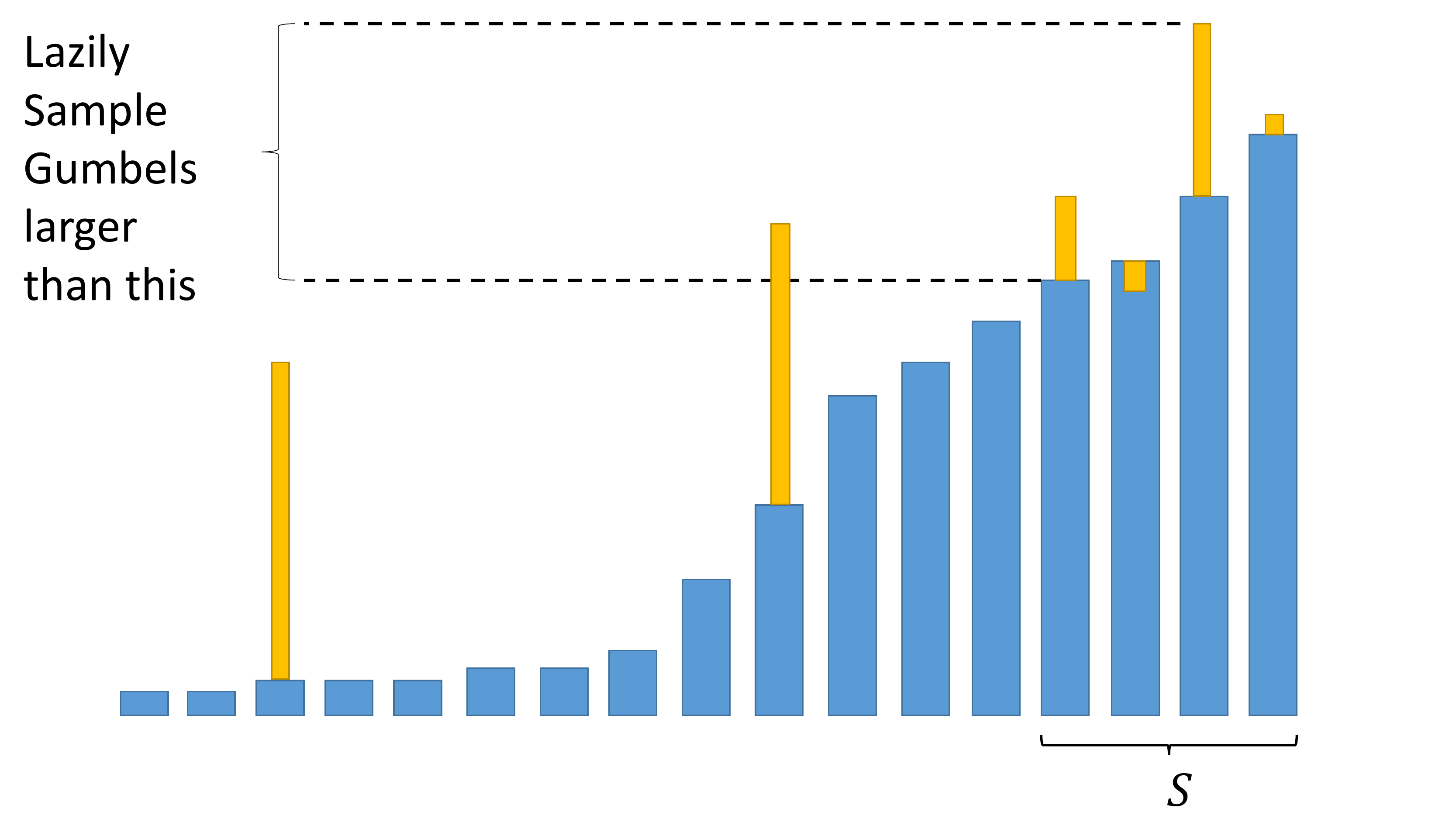}
\caption{The values of $y_i$ are shown sorted in blue while the Gumbel noise is shown in yellow. We sample a Gumbel for each element of the set $S$ of the largest $y_i$. Then, we compute the minimum value that a Gumbel must have to yield a candidate solution, represented by the difference between the dotted lines. Finally, we lazily sample Gumbels larger than this value for elements not in $S$.}
\label{fig:lazy_sample}
\end{figure}

Note that our method requires the top $k=O(\sqrt{n})$ elements of $\{y_i\}_{i=1}^n$ which we will refer to as $S$. For lazy sampling the large Gumbels, we will use the fact that a Gumbel can be represented as $G_i = -\ln(-\ln(U_i))$. Then we can sample the number of Gumbels that exceed a threshold $B$ by sampling the number of $U_i$ such that $U_i > \exp(-\exp(-B))$ and then can conditionally sample $U_i > \exp(-\exp(-B))$. Precisely, our method involves several steps shown in Algorithm \ref{alg:sampling}. 

\begin{algorithm}[tb]
   \caption{Fast Sampling with Lazy Gumbels}
   \label{alg:sampling}
\begin{algorithmic}
  \STATE {\bfseries Input:} $\{y_i\}_{i=1}^n$, $S$ as the top $k$ values of $y_i$
   \STATE Sample $k$ Gumbel variables $G_i$ for $i \in S$
   \STATE Compute $M = \max_{i \in S} y_i + G_i$ 
   \STATE Compute $S_{\mathrm{min}} = \min_{i \in S} y_i$
   \STATE Compute the Gumbel cutoff $B = M - S_{\mathrm{min}}$.
   \STATE Sample $m \sim \mathrm{Binomial}(n-k, 1 - \exp(-\exp(-B))$ as the number of $|\mathcal{X} \setminus S|$ Gumbels with value $>B$
   \STATE Uniformly sample $m$ points from $\mathcal{X} \setminus S$ and denote $T$
   \STATE Sample Gumbels that are conditionally $G_i > B$ for the points $i \in T$ (sample $U_i \sim \mathrm{Uniform}(\exp(-\exp(-B)),1)$)
   \STATE $\hat{x} = \argmax_{i \in S \cup T} y_i + G_i$
   \STATE {\bfseries return} Sample $\hat{x}$
\end{algorithmic}
\end{algorithm}

\begin{theorem}
\label{thm:correctness}
For Algorithm \ref{alg:sampling}, $\hat{x}$ is an exact sample from $\Pr(i) \propto e^{y_i}$.
\end{theorem}
\begin{proof}
This theorem would follow from Proposition \ref{prop:gumbel-max} if we prove that we are finding the maximum of $y_i + G_i$. Note that we do not evaluate the Gumbel for all of the elements in $\mathcal{X} - S - T$. Thus, the only way the lazy sampling strategy will fail is if one of these points is the true maximum. However, these points have Gumbel $G_i<B$ and since they aren't in $S$, $y_i<S_{\mathrm{min}}$. Together, this implies that $y_i + G_i < S_{\mathrm{min}} + B = M$ which is a value attained by a point in $S$. Therefore points not in $S \cup T$ cannot be the maximum.
\end{proof}

\subsubsection{Runtime}

Further, the runtime will be composed of two parts: retrieving the top $k$ elements $S$ and the runtime of Algorithm \ref{alg:sampling}. Let the cost of retrieving the top $k$ elements be $f(n,k)$. For Algorithm \ref{alg:sampling}, including the cost of retrieving $S$, the runtime will be $O(f(n,k) + m)$ and $m$ has a reasonable expected value.

\begin{theorem}
\label{thm:runtime}
For Algorithm \ref{alg:sampling}, $\mathbb{E}[m] \leq \frac{n}{k}$ 
\end{theorem}

The proof is in the appendix. Thus, the expected runtime for our method will be $O(f(n,k) + \frac{n}{k})$ which is sublinear if $k = \sqrt{n}$ and $f(n, \sqrt{n})$ is sublinear. 

\subsubsection{Fixed $B$}

Note that the technique above has a reasonable expected runtime but no runtime guarantees with high probability. To address this, we can fix $B$ to be a constant so that the value of $m$ is concentrated. 
Additionally, the technique shown in Algorithm \ref{alg:sampling} only works if $S_{\mathrm{min}}$ is an upper bound on elements not in $S$ which is brittle to errors in the MIPS technique. 

To address these issues, we define a related algorithm with a fixed Gumbel cutoff of $B = -\ln(-\ln(1 - l/n))$ so that there are on average $l$ Gumbel variables that exceed the cutoff. See Algorithm \ref{alg:robust_sampling}.

\begin{algorithm}[tb]
   \caption{Fast Sampling with Fixed $B$}
   \label{alg:robust_sampling}
\begin{algorithmic}
  \STATE {\bfseries Input:} $\{y_i\}_{i=1}^n$, $S$ as the top $k$ values of $y_i$, $l$
   \STATE Sample $k$ Gumbel variables $G_i$ for $i \in S$
   \STATE Set $B = -\ln(-\ln(1 - l/n))$
   \STATE Sample $m$ as the number of $|\mathcal{X} - S|$ Gumbels with value $>B$
   \STATE Uniformly sample $m$ points from $\mathcal{X} - S$ and call them $T$
   \STATE Sample Gumbels that are conditionally $G_i > B$ for the points $i \in T$
   \STATE $\hat{x} = \argmax_{i \in S \cup T} y_i + G_i$
   \STATE {\bfseries return} Sample $\hat{x}$
\end{algorithmic}
\end{algorithm}

Note that for $|S|=k$, the total runtime is $O(f(n,k) + m)$ where $f(n,k)$ is the runtime of gathering the top $k$ elements. Further $m \sim \mathrm{Binomial}(n, l/n)$ so with very high probability, $m<2l$ and the runtime is $O(f(n,k) + l)$ which will be sublinear if $f(n,k)$ and $l$ are sublinear.

\begin{theorem}
\label{thm:whp_correctness}
For Algorithm \ref{alg:robust_sampling}, the sample is an exact sample with probability $1 - \delta$ for $kl \geq n \ln(1/\delta)$. 
\end{theorem}

The proof is in the appendix. Thus, we can set $k=l \geq \sqrt{\ln(1/\delta)} \sqrt{n}$.

\subsection{PARTITION FUNCTION ESTIMATION}

Similar to sampling, we can estimate the partition function by using the top $k=O(\sqrt{n})$ elements $S$ and a uniform sample $T$ of $l=O(\sqrt{n})$ elements from the remaining elements. We combine these two sets to form an estimate of the partition function with relative error $\epsilon$. See Algorithm \ref{alg:partition}. 

\begin{algorithm}[tb]
   \caption{Partition Function Estimation}
   \label{alg:partition}
\begin{algorithmic}
  \STATE {\bfseries Input:} $\{y_i\}$, $S$ as the top $k$ values of $y_i$, $l$
   \STATE Uniformly sample $l$ elements with replacement from $[1,n] \setminus S$ and call it $T$
   \STATE $\hat{Z} = \sum_{i \in S} e^{y_i} + \frac{n - |S|}{|T|} \sum_{i \in T} e^{y_i}$
   \STATE {\bfseries return} Partition function estimate $\hat{Z}$
\end{algorithmic}
\end{algorithm}

\begin{theorem}
\label{thm:partition}
Algorithm \ref{alg:partition} returns an unbiased estimate $\hat{Z}$ and for $kl \geq \frac{2}{3} \frac{1}{\epsilon^2} n \ln(1/\delta)$, then with $1- \delta$ probability, 
\begin{equation}
\frac{|\hat{Z} - Z|}{Z} \leq \epsilon\end{equation}
\end{theorem}

The proof is in the appendix. If we set $k=l$, then the runtime is $O(\frac{1}{\epsilon} \sqrt{n} \sqrt{\ln(1/\delta)})$.

This is closely related to the heuristic presented in \citet{rastogi2015sublinear} as MIMPS. However, this is the first work that provides theoretical guarantees for the method and yields a theoretical understanding for the choice of $k$ and $l$.

\subsection{EXPECTED VALUE ESTIMATION}

In this section we show a way to estimate an expected value with respect to the distribution $\Pr(i) \propto e^{y_i}$. In particular, for bounded function values $\{f_i\}_{i=1}^n$ where $|f_i| \leq C$ we can define the expectation
\begin{equation}
F = \sum_i \frac{e^{y_i}}{Z} f_i
\end{equation}
where $Z = \sum_i e^{y_i}$. The algorithm we use to create an estimate is very similar to the partition function estimate. More specifically, we compute the largest $S$ values of $\{y_i\}_{i=1}^n$ and then draw uniform samples from the remaining elements and call it $T$. Then we compute an expected value using $S$ and $T$ (and upweighting the estimate from $T$). See Algorithm \ref{alg:expectation}.

\begin{algorithm}[tb]
   \caption{Expectation Estimation}
   \label{alg:expectation}
\begin{algorithmic}
  \STATE {\bfseries Input:} $\{y_i\}$, bounded function values $f_i$, $S$ as the top $k$ values of $y_i$, $l$
   \STATE Uniformly sample $l$ elements with replacement from $[1,n] \setminus S$ and call it $T$
   \STATE $\hat{Z} = \sum_{i \in S} e^{y_i} + \frac{n - |S|}{|T|} \sum_{i \in T} e^{y_i}$
   \STATE $\hat{J} = \sum_{i \in S} e^{y_i} f_i + \frac{n - |S|}{|T|} \sum_{i \in T} e^{y_i} f_i$
   \STATE $\hat{F} = \hat{J} / \hat{Z}$
   \STATE {\bfseries return} Expectation estimate $\hat{F}$
\end{algorithmic}
\end{algorithm}

This algorithm comes with a guarantee on the additive error.

\begin{theorem}
\label{thm:expectation}
Algorithm \ref{alg:expectation} returns an  estimate $\hat{F}$ such that $|\hat{F} - F| \leq \epsilon C$ with probability $\delta$ if
\begin{equation}
k^2 l \geq \frac{8n^2}{\epsilon^2} \log(4/\delta)
\end{equation}
\begin{equation}
kl \geq \frac{8}{3} \frac{1}{\epsilon^2} n \ln ( 2 / \delta )
\end{equation}
\end{theorem}

The proof is in the appendix. If we set $k = l$ then 
\begin{equation}
k = O( n^{2/3} (1/\epsilon) \sqrt{\log(1/\delta)})
\end{equation}
Then, with a sublinear MIPS technique, the total runtime is sublinear. Note that we can use this to compute the expectation of $\phi(x)$ and thus the gradient of data likelihood. This technique will be used in the experiments section for the learning experiment.

\subsection{APPROXIMATE TOP ELEMENTS}

Many Maximum Inner Product Search (MIPS) methods, including LSH-based techniques, do not solve the exact nearest neighbor problem, but approximate nearest neighbor problem. In this work, we define a similar concept of the approximate top $O(\sqrt{n})$ elements that will suffice for our theoretical arguments. Further, we show that we can use LSH instances to retrieve the approximate top $k$ elements in sublinear time. 

We say that an algorithm returns the approximate top $k$ if the gap between the smallest element in $S$ and the largest element \textit{not} in $S$ is bounded by a constant.
\begin{definition}[Approximate Top $k$]
A set of elements $S$ is an approximate top $k$ if $|S|=k$ and
\begin{equation}
\max_{i \not\in S} y_i - \min_{i \in S} y_i < c 
\end{equation}
\end{definition}
We can create a sequence of LSH instances that are ``tuned'' to a range of similarity values. Then at query time, we can go through the LSH instances in decreasing order of tuned value, gathering elements until we have $k$ elements. It turns out that these elements will be the approximate top $k$ elements (more details in the appendix). This technique will have a total runtime of 
\begin{equation}
O( k + (\log(k) + \log(1/\delta)) \log(n) n^\rho)
\end{equation}
where $\rho < 1$. Thus, we have a sublinear approximate top $k$ element MIPS technique. We state this as a theorem and prove it in the appendix.

\begin{theorem}
\label{thm:lsh_top}
For sublinear $k$, there exists a MIPS technique that returns the approximate top $k$ elements in sublinear amortized time.
\end{theorem}

 Note that if we have a MIPS technique that returns an approximate top $k$ set $S$ then we can adapt Algorithm \ref{alg:sampling} to make $B = M - S_{\mathrm{min}} - c$ for an added increase of $e^{c}$ in the expected value of $m$, and thus the runtime.

If we have a MIPS technique that returns an approximate top $k$ set $S$ with constant $c$, then Algorithm \ref{alg:robust_sampling} and \ref{alg:partition} will have an extra factor of $e^{c/2}$ for $k$ and $l$ and Algorithm \ref{alg:expectation} will have an extra factor of $e^{2c/3}$ for $k$ and $l$. These extensions are proved in the appendix and the previously stated theorems are special cases with $c=0$.

\section{EXPERIMENTS}
\label{sec:experiments}

In this section, we present an empirical evaluation of our proposed sampling, inference, and expectation techniques. The use case for our method is when there are fixed feature vectors $\{\phi(x)\}_{x\in\mathcal{X}}$, and a sequence of inference or sampling queries with different parameter vectors $\{\theta_i\}$. Although we cannot achieve gains on a single query, through preprocessing we can decrease the amortized query time. We will evaluate \textbf{runtime improvements} and \textbf{accuracy}.

\subsection{PRELIMINARIES}

\subsubsection{MIPS technique} 

We present the MIPS technique used to retrieve the top-$k$ values of the unnormalized log-probabilities. We follow the approximate nearest neighbor search method presented in \cite{douze2016polysemous} as well as the publicly available implementation. However, we will not be making use of the compression component, as we do not optimize for memory usage.

This method relies on the use of a $k$-means clustering. With the same notations as \ref{subsec:mips}, given a query $q$ and a set of vectors $V$, we aim at finding the $k$ highest values of $\{ q\cdot v, v \in V\}$. 

We first cluster the vectors in $V$ in $n_c$ clusters. For an incoming query vector $q$, we look at the inner product with the vectors in the cluster $q$ is assigned to as well as $n_p$ neighboring clusters. While this method doesn't have any theoretical guarantees, it has been shown to perform better than LSH in practice as it more advantageously exploits the distribution of the set of vectors.

In our experiments, we use CPU implementations of all algorithms for fair comparison.

\subsubsection{Data}
We experiments with two datasets from different domains to demonstrate the effectiveness of our method in real-world use.
\paragraph{Word Embeddings} We use a set of word embeddings released by Facebook \citep{bojanowski2016enriching}. Each embedding is a dense vector representing a word in a given vocabulary. These continuous representations are obtained by training log-bilinear models on large text corpora. The embeddings incorporate structure from character $n$-grams of the words. We retain words containing only letters and scale each vector to be of unit-norm. The data is composed of $N=2,000,126$ vectors of dimension $d=300$.

\paragraph{ImageNet} The ImageNet dataset \citep{russakovsky2015imagenet} from the ILSVRC 2012 competition contains 
$1.2$ million natural images divided into $1000$ classes. We extract features using a pre-trained residual network~\citep{he2016deep} trained on this classification task. More precisely, we represent each image by its activation map from the last layer before the linear classification layer of a ResNet-152. The extracted features are of size $7\times 7 \times 2048$ for each image. We then take the average along the depth dimension and reduce dimensionality using a PCA. We scale each vector to be of unit-norm. The data is thus composed of $N=1,281,167$ vectors of dimension $d=256$. In the rest of our experiments, we choose the temperature of the log-linear model to be $\tau=0.05$.

\subsection{SAMPLING}

In this section, we measure the performance of our method in terms of both sampling quality and speed. We first present empirical results on sampling and then illustrate the efficiency of our method on a specific task: a random walk over ImageNet.

\subsubsection{Sampling}

\paragraph{Speed} We want to evaluate the runtime of our method for sampling on large datasets. Given a dataset $\mathcal{X}$ and a parameter vector $\theta$, we compare the time necessary to sample from $\Pr(x) \varpropto e^{\theta \cdot \phi(x)}$ using our method or by enumeration (brute force). We compute the sampling time for random vectors $\{\theta_i\}_{i\leq 1000}$ and subsets of varying size for ImageNet ranging from $10,000$ to $1,280,000$. The results are presented in Figure~\ref{fig:sampling-speed}.

\begin{figure}
\centering
\includegraphics[width=0.5\textwidth]{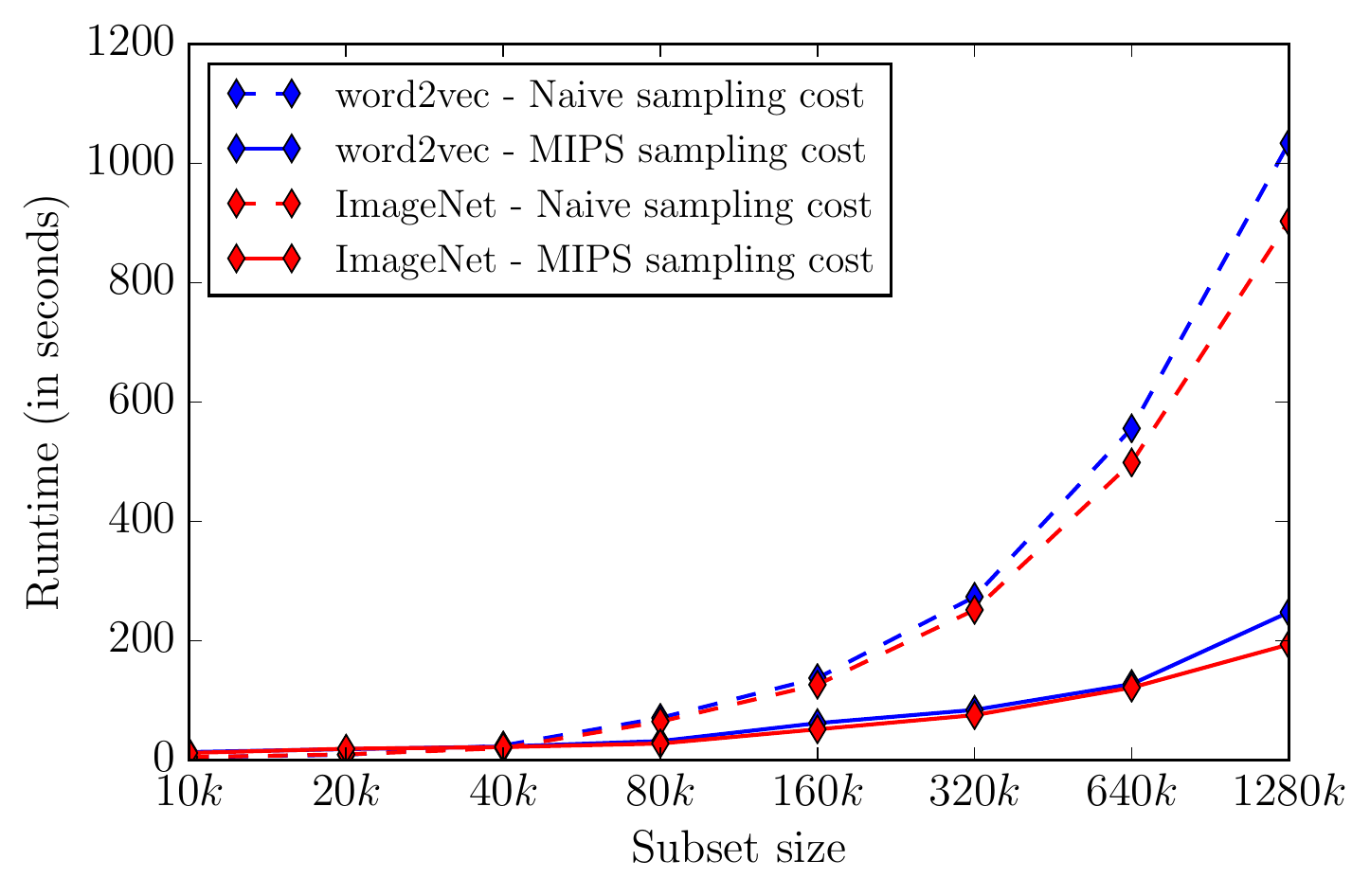}
\caption{Empirical comparison of the runtime of sampling for $10,000$ randomly chosen $\theta$ from a log-linear model on subsets (of varying size) of the datasets. Note the log-scale of the dataset size. This time is the per query runtime and does not include preprocessing.}
\label{fig:sampling-speed}
\end{figure}

We can see that the speedup is linear w.r.t the log of the sub-sampled dataset size, achieving up to $5\times$ sampling speedup for the full dataset of size $1,281,167$.
If we consider the amortized cost, i.e. including the pre-processing cost of our MIPS data structure, our method starts paying off after approximately $8,600$ samples. The amortized costs are presented in the appendix, in Figure~\ref{fig:amortized}.

\paragraph{Accuracy} To measure the accuracy of our method, we present a way to establish an upper bound on the total variation distance in closed form for a given $\theta$. Then, we average this upper bound over $100$ samples of $\theta$ (drawn uniformly from the dataset). 

Note that the lazy sampling strategy is exact unless the true maximum is not in $S \cup T$. Thus, if we can upper bound this probability, it is an upper bound on the total variation distance. For a given threshold $x$, we can compute the closed form probability that $\max_{i \not\in S \cup T} y_i + G_i < x$ and $\max_{i \in S} y_i + G_i > x$. This is the upper bound that we desire and we can optimize $x$ for the tightest upper bound. For both datasets, over 100 samples of $\theta$, the average upper bound was on the order of $10^{-4}$ proving that our sampling method is accurate even while using an approximate MIPS technique. A summary of our results in terms of accuracy and speedup are provided in Table~\ref{tab:speed-accuracy}. We provide further empirical evidence in the appendix to show that the distributions closely match on the shown $\theta$.

\begin{table}
\centering
\begin{tabular}{l c c}
\toprule
Dataset & Speedup & Total Variation Bound \\
\midrule
ImageNet & $4.65\times$ & $(2.5 \pm 1.4)\times 10^{-4}$\\
Word Embeddings & $4.17\times$ & $(4.8 \pm 2.2) \times 10^{-4}$ \\
\bottomrule
\end{tabular}
\caption{Summary of the sampling speedup and bound on the total variation distance for our method on the ImageNet and Word Embeddings datasets.}
\label{tab:speed-accuracy}
\end{table}

\subsubsection{Random walk over a large set}

To showcase the applicability of our method, we perform a random walk over the ImageNet dataset. We define the transition function, i.e. the probability to walk from image $j$ to image $i$ as $\Pr(X_{t+1} = i | X_t = j) \propto e^{\tau \phi(x_i) \cdot \phi(x_j)}$ where $\tau$ is the temperature, $\phi$ is the fixed featurization previously defined, and $x_i, x_j$ are the pixel-values of images $i,j$.
The initial state is sampled uniformly across the dataset. This is similar in spirit to the PageRank algorithm \citep{page1999pagerank}.
This setting fits our method because while the MIPS structure can be reused across time steps, no computation can be cached in the naive setting (assuming we do not store the distribution for each element, which would be on the order of Terabytes).

\begin{figure}
\centering
\includegraphics[width=0.5\textwidth]{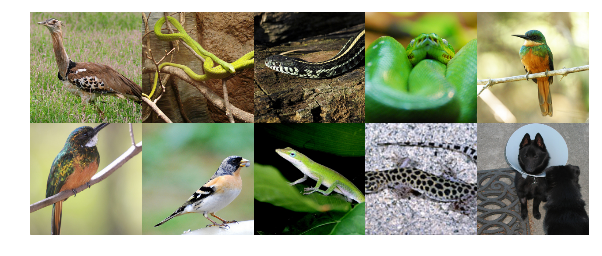}
\caption{Samples of the Markov chain. The samples are spaced out by $20$ time steps.}
\label{fig:samples-mc}
\end{figure}

We evaluate the quality of the Markov Chain by comparing the top elements of the empirical sampling distribution. We run two different Markov chains, one with exact sampling and one with our sampling technique. Over one million steps, the two Markov Chains share $73.6\%$ of the top $1000$ elements. This percentage looks low because of the finite sampling error. When we compare two different one million element windows within each chain, the top 1000 elements are shared $69.3\%$ and $72.9\%$ for the exact sampling and our sampling, respectively. It is seen that the between-chain differences are the same as the within-chain differences, so the Markov chain with our sampling technique yields roughly the same distribution as the chain with exact sampling.
\subsection{PARTITION FUNCTION ESTIMATE}
We show the performance of our partition function estimate as shown in Algorithm \ref{alg:partition}. We can trade-off error and runtime by varying $k$ and $l$. See Figure \ref{fig:partition-speed}. We average the results over several values of $\theta$, drawn uniformly from the dataset. For comparison, we plot the trade-off for only looking at the top $k$ values and using this as a partition function estimate. Additionally, we compare to the method of \cite{mussmann2016learning} for different size of noise $t$. For each value of $k, l$ and $t$ we report the runtime and relative error of the partition function estimate. As shown by the relative error of the top-$k$ estimate, sampling from the tail is necessary to achieve low relative error. We also show that the method from \cite{mussmann2016learning} cannot come close in terms of relative error, achieving a maximum of $15\%$ relative error for $t=64$. It is also important to note that their method cannot trade-off speed for accuracy as, when the noise-length $t$ increases, the injected noise destroys the MIPS structure rendering it highly inaccurate.

\begin{figure}
\centering
\includegraphics[width=\columnwidth]{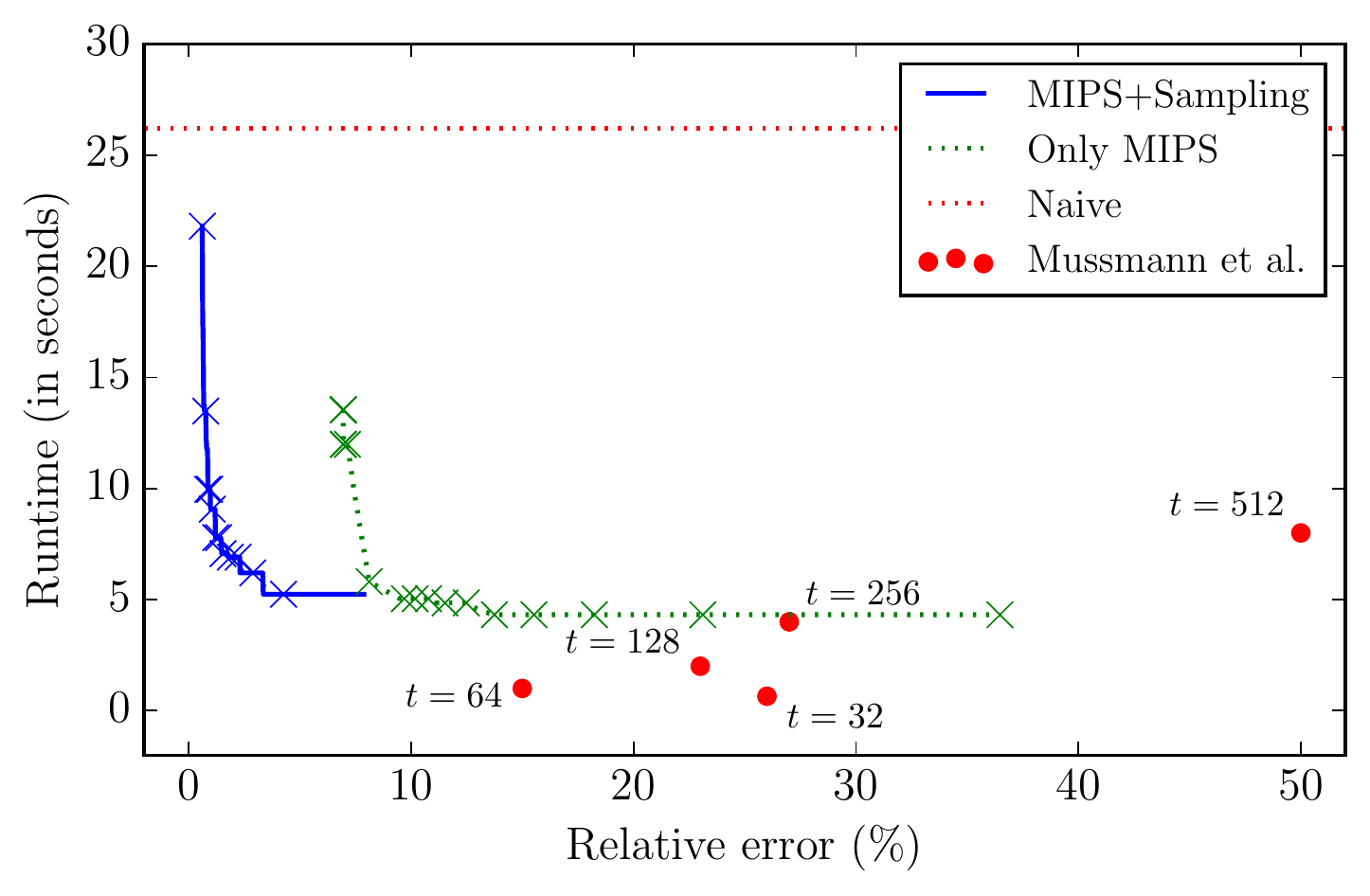}
\caption{Runtime plotted as a function of relative error of partition function estimate (different points made by varying $k$ and $l$) on ImageNet (averaged over random values of $\theta$. The red dotted line is the time for the exact partition function computation.}
\label{fig:partition-speed}
\end{figure}

\subsection{LEARNING}
We wish to maximize the likelihood of a subset of the data $\mathcal{D} \subseteq \mathcal{X}$ given $\Pr(\cdot; \theta)$. We aim at finding 
\begin{equation} \theta^* = \arg \max_\theta \sum_{x \in \mathcal{D}} \log \Pr(x; \theta)
\end{equation}
using gradient ascent. Evaluating the gradient requires finding the expectation of the features $\phi(x)$ which can be estimated using our method in Algorithm \ref{alg:expectation}.
The features are fixed but $\theta$ is updated at each step of the gradient ascent algorithm, fitting well into the setting of our method. We choose a small subset $\mathcal{D}$ of ImageNet as images with a commonality. In particular, we handpick 16 images showing the presence of water. We compare computing the gradient with our method to the computation of the exact gradient and to approximating the gradient by considering the truncated distribution on the top $k$ elements (referred to as top-$k$ gradient). The chosen images are shown in the appendix in Figure \ref{fig:training_D}. We perform gradient ascent for $5000$ iterations with learning rate $\alpha = 10$, which we halve every $1000$ iterations. The results are reported in Table~\ref{tab:gd}. The learning curves are shown in Figure~\ref{fig:learning-curve}. We also show the $10$ most probable samples (outside of the dataset $\mathcal{D}$) according to the log-linear model in Figure~\ref{fig:samples}. We can see that these images are semantically similar to the training set, all containing water, showcasing the expressive power of the ResNet features.

\begin{table}
\centering
\begin{tabular}{l c c}
\toprule
Method & Log-likelihood & Speedup \\
\midrule
Exact gradient & $-3.170$ & $1\times$ \\
Only top-$k$ & $-4.062$ & $22.7\times$ \\
Our method & $-3.175$ & $9.6\times$ \\
\bottomrule
\end{tabular}
\caption{Log-likelihood and speedup for the learning of a log-linear model on ImageNet. For our method, we picked $k=10\sqrt{n}, l = 10k$, for the comparison to only weighing the top-$k$, we chose $k=100\sqrt{n}$ as well.}
\label{tab:gd}
\end{table}

\begin{figure}[t]
\centering
\includegraphics[width=0.5\textwidth]{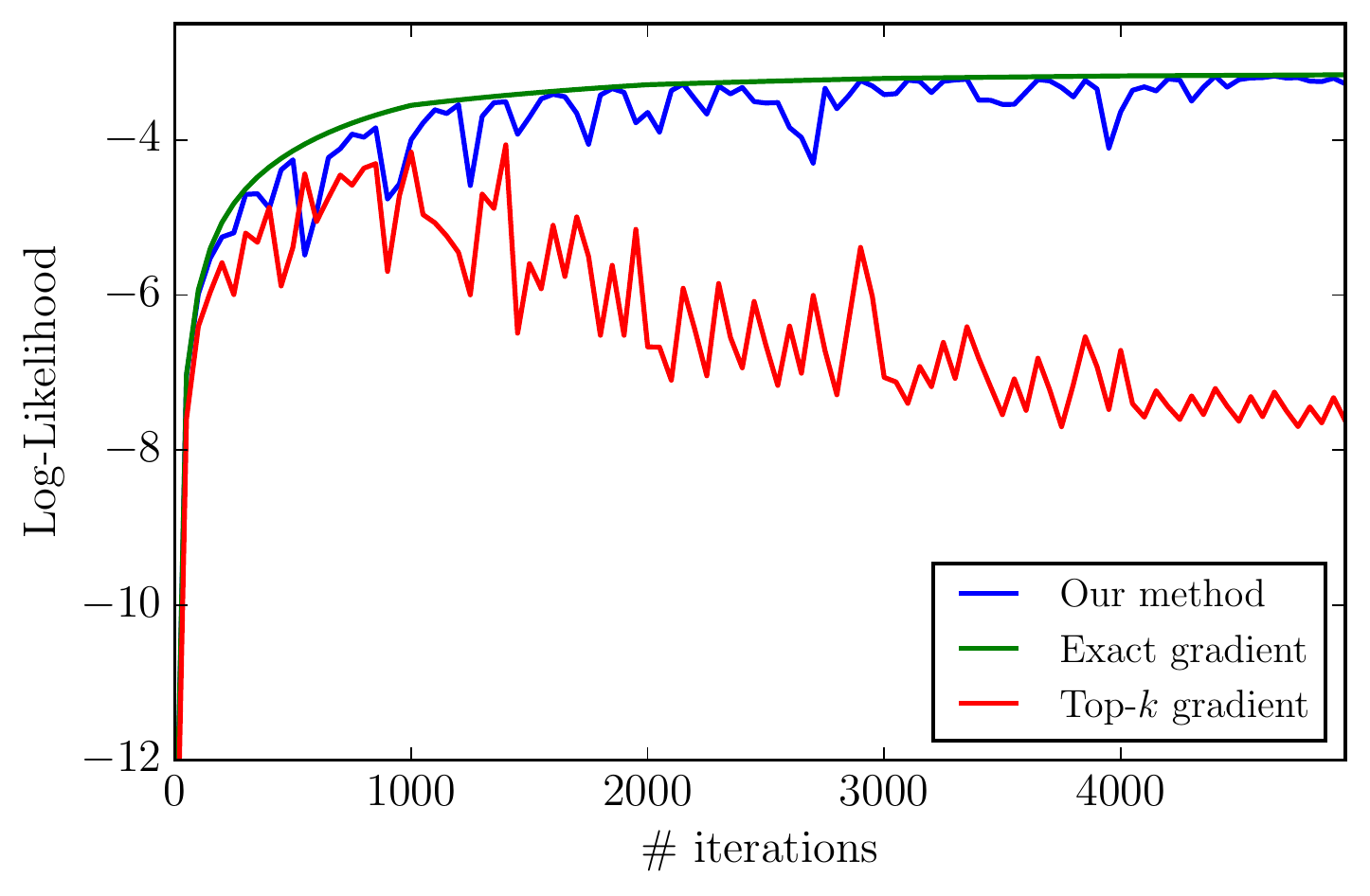}
\caption{Log-likelihood plotted against the number of iterations for performing gradient ascent on our learning problem for $5000$ iterations with a learning rate $\alpha=10$, halving the learning rate every $1000$ iterations.}
\label{fig:learning-curve}
\end{figure}

\begin{figure}[t]
\centering
\includegraphics[width=0.5\textwidth]{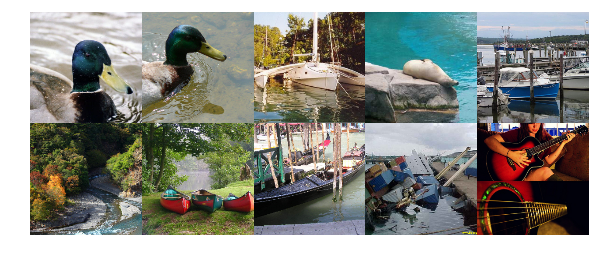}
\caption{$10$ most probable images (outside of $\mathcal{D}$) from our log-linear model trained to convergence.}
\label{fig:samples}
\end{figure}
As shown in Figure~\ref{fig:learning-curve}, we can see that the log-likelihood for our method and the exact gradient almost exactly overlap indicating that our estimation of the gradient is very accurate. In contrast, the top-$k$ gradient, while faster, proves to be a poor estimator and thus cannot optimize the log-likelihood. To summarize, \textbf{our method converges to the global maximum $\mathbf{9.6\times}$ as fast as computing the exact gradient}.

\section{RELATED WORK}
\label{sec:related}

Our method can be viewed in two different comparative perspectives. Our method can be seen as an alternative to only using the top-$k$ most probable elements as is done in \citet{vijayanarasimhan2014deep}. There, large output spaces for deep learning are handled using Locality Sensitive Hashing. In particular, the top vectors are gathered and the rest of the vectors in the tail are ignored. For spread-out distributions (closer to uniform), this method will fail. Our work provides a scalable method to incorporate the probability mass present in the tail of the distribution by sampling $O(\sqrt{n})$ elements, a small prices compared to retrieving the top elements.

Our method can also be compared to a different way of combining the Gumbel max trick and Maximum Inner Product Search as presented in \citet{mussmann2016learning}. In that work, Gumbel noise is appended to the database vectors and stored in the MIPS data structure. Then, query vectors are chosen to access the frozen Gumbel noise. That work has several major shortcomings that make it unusable in practice. 

The Gumbel noise is re-used, introducing correlated samples and systematic bias in the partition function estimate. In particular, for any fixed value of the parameters, there are a fixed number of samples ``frozen'' into the stored Gumbel noise. We avoid this issue by sampling $O(\sqrt{n})$ fresh Gumbel variables for every sample. While real world data often has structure that can be exploited by the MIPS techniques, in \citet{mussmann2016learning}, the structure is destroyed by injecting random Gumbel noise. In our technique, we preserve structure in the database vectors by leaving the vectors unchanged. Finally, the method of \citet{mussmann2016learning} requires accessing the MIPS data structure many times for independent samples and partition function estimates. In this work, we only require accessing the MIPS data structure once per parameter value.

\section{CONCLUSION}
\label{sec:conclusion}

In conclusion, we have presented several related methods that are based on the key idea of accessing the large elements in a distribution using Maximum Inner Product Search and accessing the tail of a distribution with uniform sampling. This decreases the runtime from $O(n)$ to $O(\sqrt{n})$ plus the runtime for the MIPS technique. 

This work is best suited for cases where the output space is large but enumerable, such as those in NLP and computer vision. This work can be expected to give speedups when the feature vectors of a log-linear model are fixed but it is desired to perform inference and sampling for several different values of the the parameters. Note that our method is as flexible as the MIPS method that is employed; the feature vectors need to only be fixed for the MIPS to work. As an example, if a MIPS system allows for sparse updates, our method will also allow for sparse updates. Since our method treats MIPS as a black-box, advances in the speed and accuracy of MIPS techniques automatically improve our method. 

When accessing the top elements is not accurate enough, we present a method to include uniform samples from the tail to provide provably good samples and estimates of the partition function. All this, at the small overhead price of uniform sampling.

\section{Acknowledgments}
\label{sec:acknowledgments}

This research was supported by Intel Corporation, Future of Life Institute ($\#2016-158687$) and NSF grants $1651565$, $1649208$, $1522054$, and DGE-$1656518$.

We thank Ludwig Schmidt and Moses Charikar for helpful discussions.
\subsubsection*{References}

\bibliography{bibliography}

\begin{thebibliography}{25}
\providecommand{\natexlab}[1]{#1}
\providecommand{\url}[1]{\texttt{#1}}
\expandafter\ifx\csname urlstyle\endcsname\relax
  \providecommand{\doi}[1]{doi: #1}\else
  \providecommand{\doi}{doi: \begingroup \urlstyle{rm}\Url}\fi

\bibitem[Auvolat et~al.(2015)Auvolat, Chandar, Vincent, Larochelle, and
  Bengio]{auvolat2015clustering}
Auvolat, Alex, Chandar, Sarath, Vincent, Pascal, Larochelle, Hugo, and Bengio,
  Yoshua.
\newblock Clustering is efficient for approximate maximum inner product search.
\newblock \emph{arXiv preprint arXiv:1507.05910}, 2015.

\bibitem[Bengio et~al.(2003)Bengio, Ducharme, Vincent, and
  Jauvin]{bengio2003neural}
Bengio, Yoshua, Ducharme, R{\'e}jean, Vincent, Pascal, and Jauvin, Christian.
\newblock A neural probabilistic language model.
\newblock \emph{Journal of machine learning research}, 3\penalty0
  (Feb):\penalty0 1137--1155, 2003.

\bibitem[Bentley(1975)]{bentley1975multidimensional}
Bentley, Jon~Louis.
\newblock Multidimensional binary search trees used for associative searching.
\newblock \emph{Communications of the ACM}, 18\penalty0 (9):\penalty0 509--517,
  1975.

\bibitem[Bojanowski et~al.(2016)Bojanowski, Grave, Joulin, and
  Mikolov]{bojanowski2016enriching}
Bojanowski, Piotr, Grave, Edouard, Joulin, Armand, and Mikolov, Tomas.
\newblock Enriching word vectors with subword information.
\newblock \emph{arXiv preprint arXiv:1607.04606}, 2016.

\bibitem[Charikar(2002)]{charikar2002similarity}
Charikar, Moses~S.
\newblock Similarity estimation techniques from rounding algorithms.
\newblock In \emph{Proceedings of the thiry-fourth annual ACM symposium on
  Theory of computing}, pp.\  380--388. ACM, 2002.

\bibitem[Douze et~al.(2016)Douze, J{\'e}gou, and
  Perronnin]{douze2016polysemous}
Douze, Matthijs, J{\'e}gou, Herv{\'e}, and Perronnin, Florent.
\newblock Polysemous codes.
\newblock In \emph{European Conference on Computer Vision}, pp.\  785--801.
  Springer International Publishing, 2016.

\bibitem[Gumbel \& Lieblein(1954)Gumbel and Lieblein]{gumbel1954statistical}
Gumbel, Emil~Julius and Lieblein, Julius.
\newblock Statistical theory of extreme values and some practical applications:
  a series of lectures.
\newblock 1954.

\bibitem[Hazan et~al.(2013)Hazan, Maji, and Jaakkola]{hazan2013sampling}
Hazan, Tamir, Maji, Subhransu, and Jaakkola, Tommi.
\newblock On sampling from the gibbs distribution with random maximum
  a-posteriori perturbations.
\newblock In \emph{Advances in Neural Information Processing Systems}, pp.\
  1268--1276, 2013.

\bibitem[He et~al.(2016)He, Zhang, Ren, and Sun]{he2016deep}
He, Kaiming, Zhang, Xiangyu, Ren, Shaoqing, and Sun, Jian.
\newblock Deep residual learning for image recognition.
\newblock In \emph{Computer Vision and Pattern Recognition (CVPR), 2016 IEEE
  Conference on}, 2016.

\bibitem[Indyk \& Motwani(1998)Indyk and Motwani]{indyk1998approximate}
Indyk, Piotr and Motwani, Rajeev.
\newblock Approximate nearest neighbors: towards removing the curse of
  dimensionality.
\newblock In \emph{Proceedings of the thirtieth annual ACM symposium on Theory
  of computing}, pp.\  604--613. ACM, 1998.

\bibitem[Joulin et~al.(2016)Joulin, van~der Maaten, Jabri, and
  Vasilache]{joulin2016learning}
Joulin, Armand, van~der Maaten, Laurens, Jabri, Allan, and Vasilache, Nicolas.
\newblock Learning visual features from large weakly supervised data.
\newblock In \emph{European Conference on Computer Vision}, pp.\  67--84.
  Springer, 2016.

\bibitem[Kim et~al.(2016)Kim, Sabharwal, and Ermon]{kim2016exact}
Kim, Carolyn, Sabharwal, Ashish, and Ermon, Stefano.
\newblock Exact sampling with integer linear programs and random perturbations.
\newblock In \emph{Proc. 30th AAAI Conference on Artificial Intelligence},
  2016.

\bibitem[Koenigstein et~al.(2012)Koenigstein, Ram, and
  Shavitt]{koenigstein2012efficient}
Koenigstein, Noam, Ram, Parikshit, and Shavitt, Yuval.
\newblock Efficient retrieval of recommendations in a matrix factorization
  framework.
\newblock In \emph{Proceedings of the 21st ACM international conference on
  Information and knowledge management}, pp.\  535--544. ACM, 2012.

\bibitem[Koller \& Friedman(2009)Koller and Friedman]{koller2009probabilistic}
Koller, Daphne and Friedman, Nir.
\newblock \emph{Probabilistic graphical models: principles and techniques}.
\newblock MIT press, 2009.

\bibitem[Maddison et~al.(2014)Maddison, Tarlow, and
  Minka]{maddison2014sampling}
Maddison, Chris~J, Tarlow, Daniel, and Minka, Tom.
\newblock A* sampling.
\newblock In \emph{Advances in Neural Information Processing Systems}, pp.\
  3086--3094, 2014.

\bibitem[Mikolov et~al.(2013)Mikolov, Sutskever, Chen, Corrado, and
  Dean]{mikolov2013distributed}
Mikolov, Tomas, Sutskever, Ilya, Chen, Kai, Corrado, Greg~S, and Dean, Jeff.
\newblock Distributed representations of words and phrases and their
  compositionality.
\newblock In \emph{Advances in neural information processing systems}, pp.\
  3111--3119, 2013.

\bibitem[Murphy(2012)]{murphy2012machine}
Murphy, Kevin~P.
\newblock \emph{Machine learning: a probabilistic perspective}.
\newblock MIT press, 2012.

\bibitem[Mussmann \& Ermon(2016)Mussmann and Ermon]{mussmann2016learning}
Mussmann, Stephen and Ermon, Stefano.
\newblock Learning and inference via maximum inner product search.
\newblock In \emph{Proceedings of The 33rd International Conference on Machine
  Learning}, pp.\  2587--2596, 2016.

\bibitem[Neyshabur \& Srebro(2014)Neyshabur and Srebro]{neyshabur2014symmetric}
Neyshabur, Behnam and Srebro, Nathan.
\newblock On symmetric and asymmetric lshs for inner product search.
\newblock \emph{arXiv preprint arXiv:1410.5518}, 2014.

\bibitem[Page et~al.(1999)Page, Brin, Motwani, and Winograd]{page1999pagerank}
Page, Lawrence, Brin, Sergey, Motwani, Rajeev, and Winograd, Terry.
\newblock The pagerank citation ranking: Bringing order to the web.
\newblock Technical report, Stanford InfoLab, 1999.

\bibitem[Ram \& Gray(2012)Ram and Gray]{ram2012maximum}
Ram, Parikshit and Gray, Alexander~G.
\newblock Maximum inner-product search using cone trees.
\newblock In \emph{Proceedings of the 18th ACM SIGKDD international conference
  on Knowledge discovery and data mining}, pp.\  931--939. ACM, 2012.

\bibitem[Rastogi \& Van~Durme(2015)Rastogi and Van~Durme]{rastogi2015sublinear}
Rastogi, Pushpendre and Van~Durme, Benjamin.
\newblock Sublinear partition estimation.
\newblock \emph{arXiv preprint arXiv:1508.01596}, 2015.

\bibitem[Russakovsky et~al.(2015)Russakovsky, Deng, Su, Krause, Satheesh, Ma,
  Huang, Karpathy, Khosla, Bernstein, et~al.]{russakovsky2015imagenet}
Russakovsky, Olga, Deng, Jia, Su, Hao, Krause, Jonathan, Satheesh, Sanjeev, Ma,
  Sean, Huang, Zhiheng, Karpathy, Andrej, Khosla, Aditya, Bernstein, Michael,
  et~al.
\newblock Imagenet large scale visual recognition challenge.
\newblock \emph{International Journal of Computer Vision}, 115\penalty0
  (3):\penalty0 211--252, 2015.

\bibitem[Shrivastava \& Li(2014)Shrivastava and Li]{shrivastava2014asymmetric}
Shrivastava, Anshumali and Li, Ping.
\newblock Asymmetric lsh (alsh) for sublinear time maximum inner product search
  (mips).
\newblock In \emph{Advances in Neural Information Processing Systems}, pp.\
  2321--2329, 2014.

\bibitem[Vijayanarasimhan et~al.(2014)Vijayanarasimhan, Shlens, Monga, and
  Yagnik]{vijayanarasimhan2014deep}
Vijayanarasimhan, Sudheendra, Shlens, Jonathon, Monga, Rajat, and Yagnik, Jay.
\newblock Deep networks with large output spaces.
\newblock \emph{arXiv preprint arXiv:1412.7479}, 2014.

\end{thebibliography}

\clearpage
\section{APPENDIX}

\subsection{EMPIRICAL EVALUATION OF SAMPLING}
We wish to evaluate the empirical accuracy of our sampling technique on concrete examples. We do this in two ways. First, we can sort the elements by probability and make events of drawing an element in the top $10$, or the top $100$, top $1000$, etc. We show the results for two random $\theta$ with different distributions in Figure \ref{fig:accuracy} for $50,000$ samples. Note that our method closely matches the histogram of the true distribution. For a more comprehensive evaluation, we sample 30 values of $\theta$ and compute the relative error for exact sampling and our approximate sampling. See Figure \ref{fig:accuracy}. We also present in Figure~\ref{fig:amortized} the amortized speedups obtained by our method. The amortized cost is defined as the time needed to train the index, added to the runtime of sampling $10,000$ samples. 

\begin{figure*}
\centering
\includegraphics[width=\textwidth]{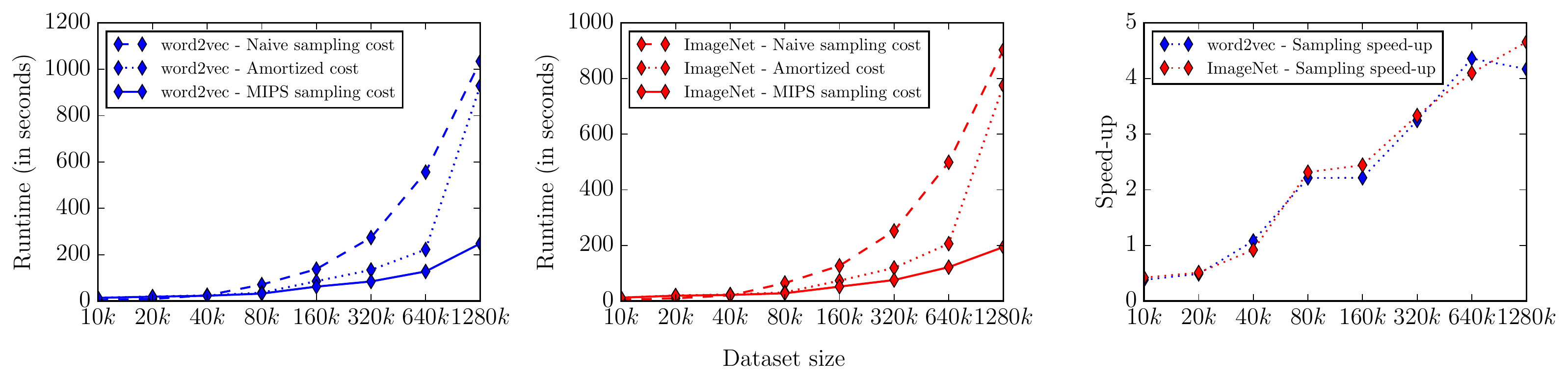}
\caption{\textit{Left and Center:} empirical comparison of the runtime of sampling for a $10,000$ randomly chosen $\theta$ on Word Embeddings and Image Net for varying fraction of the data. The amortized cost is defined as the time necessary for the sampling in addition to the training time of the index. \textit{Right:} Evaluation of the sampling speed-up for both datasets for varying fraction of the data.}
\label{fig:amortized}
\end{figure*}

\begin{figure*}
\centering
\includegraphics[width=\textwidth]{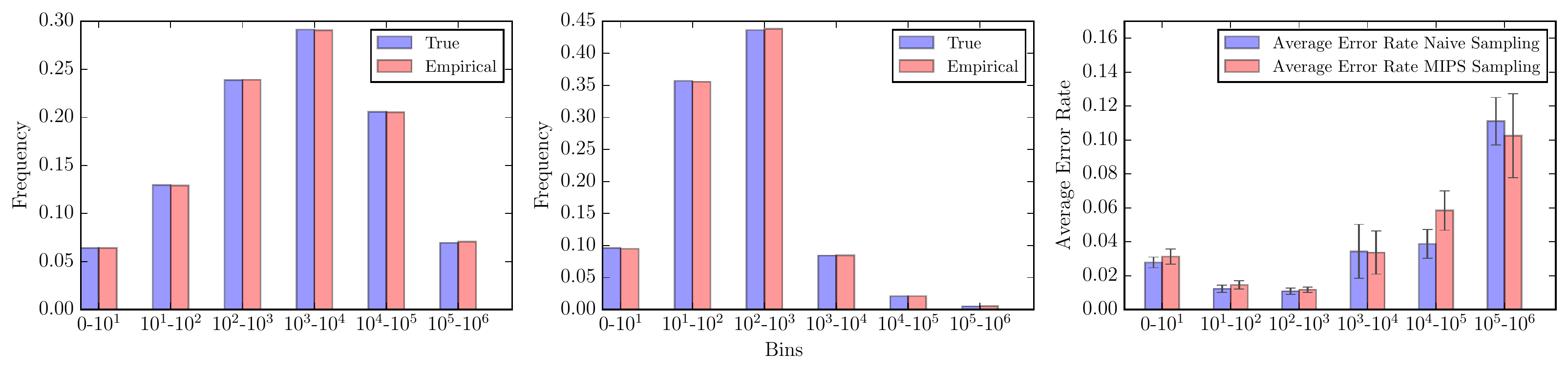}
\caption{\textit{Left and Center:} Two randomly chosen $\theta$ with different bin distributions. We see that the empirical sampling closely matches the true distribution for all the bins. \textit{Right:} Evaluation of the relative error on 30 samples of $\theta$ for the exact sampling and our sampling technique. The error bars are for the average error rate between the empirical distribution and the true distribution for both exact sampling and our method. We see that the error rates are not statistically significantly different.}
\label{fig:accuracy}
\end{figure*}

\subsection{LEARNING TRAINING SET}
For the learning experiment, we show the set of images $\mathcal{D}$ that we maximized the probability of. See Figure \ref{fig:training-images}. The common theme of the images is the presence of water.

\begin{figure}
\centering
\includegraphics[width=0.45\textwidth]{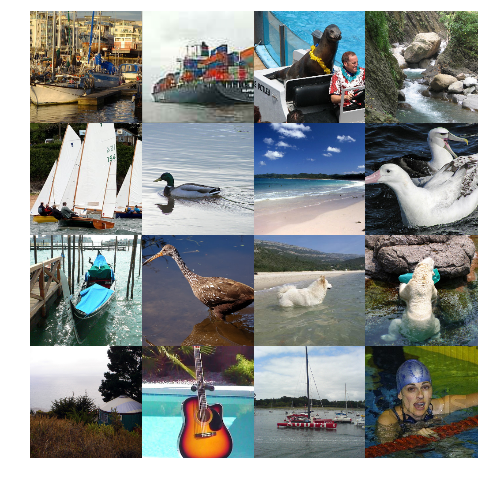}
\caption{The set of images $\mathcal{D}$ used in the learning experiment. Note that all of the images contain water, though the content of the images is quite different.}
\label{fig:training-images}
\label{fig:training_D}
\end{figure}

\subsection{VALUE OF $c$}
For all of the proofs, we will include the result with the approximate MIPS with an error of $c$. To recover the original results in the paper, set $c=0$.

\subsection{SAMPLING}

\begin{thmn}[\ref{thm:runtime}]
For Algorithm \ref{alg:sampling}, $\mathbb{E}[m] \leq \frac{n e^c}{k}$ 
\end{thmn}
\begin{proof}
Note that $m$ is the number of Gumbels that are larger than $B = M - S_{min} - c$. 

Note that Gumbels can be defined by $-\ln(-\ln(U_i))$ where $U_i$ is a uniform random variable on the interval $[0,1]$. Thus, we can think of each point having a uniform sample $U_i$ and finding places where

\begin{equation}
-\ln(-\ln(U_i)) > M - S_{min} - c
\end{equation}

\begin{equation}
U_i > \exp(-\exp(S_{min} + c - M))
\end{equation}

Thus, if we can find places where $U_i > \exp(-\exp(S_{min} + c - M))$, then we have the value of $m$. The number of points where this occurs is distributed according to $Bin(n - |S|, 1 - \exp(-\exp(S_{min} + c - M))$. 

Thus, 

\begin{equation}
\mathbb{E}[m|M] = (n - |S|) (1 - \exp(-\exp(S_{min} + c - M)))
\end{equation}
\begin{equation}
\mathbb{E}[m|M] \leq n \exp(S_{min} + c - M)
\end{equation}

Note that

\begin{equation}
\Pr[ne^{S_{min} + c - M} > x] = \Pr[ M - S_{min} - c < \ln(n/x)]
\end{equation}

\begin{equation}
= \Pr[(\max_{i \in S} y_i + G_i) - S_{min} - c < \ln(n/x)]
\end{equation}

\begin{equation}
\leq \Pr[\max_{i \in S} G_i - c < \ln(n/x)]
\end{equation}

\begin{equation}
\leq \Pr[\ln(|S|) + G - c < \ln(n/x)]
\end{equation}

\begin{equation}
\leq \Pr[\frac{n e^c e^{-G}}{|S|} > x]
\end{equation}

\begin{equation}
\leq \Pr[exponential(\frac{n e^c}{|S|}) > x]
\end{equation}

And thus,

\begin{equation}
\mathbb{E}[ne^{S_{min} + c - M}] \leq \mathbb{E}[exponential(\frac{n e^c}{|S|})] = \frac{n e^c}{|S|}
\end{equation}

Putting it all together,

\begin{equation}
\mathbb{E}[m] = \mathbb{E}[\mathbb{E}[m|M]] \leq \frac{n e^c}{|S|}
\end{equation}

\end{proof}

\begin{thmn}[\ref{thm:whp_correctness}]
For Algorithm \ref{alg:robust_sampling}, the sample is an exact sample with probability $1 - \delta$ for $\delta = \exp(-\frac{kl}{n} e^{-c})$. 
\end{thmn}
\begin{proof}
Note that the elements not in $S \cup T$ have values $y_i \leq S_{min} + c$ and $G_i \leq B = -\ln(-\ln(1 - l/n))$. Further, there exists an element in $S$ with $y_i \geq S_{min}$ and with $G_i = \max_{j=1}^k G_j$. As long as there is an element in $S$ that exceeds all the elements not in $S \cup T$, the sample will be exact.

\begin{equation}
\Pr[\text{not exact sample}] \leq
\end{equation}
\begin{equation}
\leq \Pr[\max_{j=1}^k G_j < -\ln(-\ln(1 - l/n)) + c]
\end{equation}
\begin{equation}
\leq \Pr[\ln(k) - \ln(-\ln(U)) < -\ln(-\ln(1 - l/n)) + c]
\end{equation}
\begin{equation}
\leq \Pr[k(-\ln(1 - l/n)) e^{-c} < -\ln(U)]
\end{equation}
\begin{equation}
\leq \Pr[\frac{kl}{n} e^{-c} < -\ln(U)]
\end{equation}
\begin{equation}
\leq \Pr[\exp(-\frac{kl}{n} e^{-c}) > U]
\end{equation}
\begin{equation}
\leq \exp(-\frac{kl}{n} e^{-c})
\end{equation}

Thus, the probability of failure, $\delta$, is bounded by $\exp(-\frac{kl}{n} e^{-c})$
\end{proof}

\subsection{PARTITION FUNCTION ESTIMATE}

\begin{thmn}[\ref{thm:partition}]
Algorithm \ref{alg:partition} returns an unbiased estimate $\hat{Z}$ and for $kl \geq \frac{2}{3} \frac{1}{\epsilon^2} n e^c \ln(1/\delta)$, then with $1 - \delta$ probability, 

$$\frac{|\hat{Z} - Z|}{Z} \leq \epsilon$$

\end{thmn}

\begin{proof}
Define $Z = \sum_i e^{y_i}$. Let $S$ be the indices of the $k$ largest elements of $\{y_i\}$ and $S' = [1,n] \setminus S$. Denote $||S'||_1 = \sum_{i \in S'} e^{y_i}$ and $||S||_1 = \sum_{i \in S} e^{y_i}$. Thus, the true partition function is $Z = \sum_{i \in S} e^{y_i} + \sum_{i \in S'} e^{y_i} = ||S||_1 + ||S'||_1$.

Let $r$ be the largest value of $e^{y_i}$ for elements in $S'$. Then for all $i \in S'$, $e^{y_i} \in (0, r]$ and further, for all $i \in S$, $e^{y_i} \geq r e^{-c}$. We can scale these values of $S'$ and denote them as $q_i = \frac{e^{y_i}}{r}$ where $q_i \in [0,1]$.

For the estimate $\hat{Z}$ we will draw $l$ samples with replacement from $S'$ and denote the set as $T$. Denote the samples as $y^{(j)}$ and the scaled versions as $q^{(j)} = \frac{e^{y^{(j)}}}{r}$. 

We use the estimate:

\begin{equation}
\hat{Z} = \sum_{i \in S} e^{y_i} + \frac{|\mathcal{X} - S|}{|T|} \sum_{i \in T} e^{y_i}
\end{equation}
\begin{equation}
\hat{Z} = ||S||_1 + \frac{(n-k)r}{l} \sum_{j = 1}^l q^{(j)}
\end{equation}

Note that

\begin{equation}
\mathbb{E}[\hat{Z}] = (n-k)r \mathbb{E}[q^{(1)}] + ||S||_1
\end{equation}
\begin{equation}
\mathbb{E}[\hat{Z}] = (n-k) r \sum_{i \in S'} \frac{1}{|S'|} \frac{e^{y_i}}{r} + ||S||_1 = ||S'||_1 + ||S||_1 = Z
\end{equation}

This is because $|S'| = n-k$. Thus, $\hat{Z}$ is an unbiased estimator of $Z$. However, we are concerned if it is well concentrated about its mean.

Let $Q$ be a random variable as the scaled sample from $S'$ and $\bar{Q}$ be the empirical mean over $l$ samples. Thus, $\mathbb{E}[Q] = \frac{||S'||_1}{(n-k) r}$.

Note that 

\begin{equation}
|\hat{Z} - Z| = |\frac{(n-k)r}{l} \sum_{j = 1}^l q^{(j)} - ||S'||_1 |
\end{equation}
\begin{equation}
|\hat{Z} - Z| = |\frac{(n-k)r}{l} \sum_{j = 1}^l q^{(j)} - (n-k) r \mathbb{E}[Q] |
\end{equation}
\begin{equation}
|\hat{Z} - Z| = (n-k) r |\frac{1}{l} \sum_{j = 1}^l q^{(j)} - \mathbb{E}[Q] |
\end{equation}
\begin{equation}
|\hat{Z} - Z| = (n-k) r |\bar{Q} - \mathbb{E}[Q] |
\end{equation}

Therefore,

\begin{equation}
\Pr[|\hat{Z} - Z| > \epsilon Z] = \Pr[ (n-k)r |\bar{Q} - \mathbb{E}[Q] | > \epsilon (||S'||_1 + ||S||_1)]
\end{equation}
\begin{equation}
 = \Pr[ (n-k)r |\bar{Q} - \mathbb{E}[Q] | > \epsilon ((n-k) r \mathbb{E}[Q] + ||S||_1)]
\end{equation}
\begin{equation}
 = \Pr[ |\bar{Q} - \mathbb{E}[Q] | > \epsilon (\mathbb{E}[Q] + \frac{||S||_1}{(n-k)r} )]
\end{equation}
\begin{equation}
 \leq \Pr[ |\bar{Q} - \mathbb{E}[Q] | > \epsilon (\mathbb{E}[Q] + \frac{k e^{-c}}{n} )]
\end{equation}

If we use Chernoff (use a convexity argument to bound the MGF in terms of the mean as on page 22 of "Concentration of Measure for the Analysis of Randomised Algorithms" by Dubhashi and Panconesi)
\begin{equation}
\Pr[|\sum_j Q^{(j)} - l\mathbb{E}[Q]| > \delta l\mathbb{E}[Q] ] \leq 2 \exp(-\frac{1}{3} \delta^2 l\mathbb{E}[Q] )
\end{equation}
\begin{equation}
\Pr[|\bar{Q} - \mathbb{E}[Q]| > \delta \mathbb{E}[Q] ] \leq 2 \exp(-\frac{1}{3} \delta^2 l\mathbb{E}[Q] )
\end{equation}
\begin{equation}
\Pr[|\bar{Q} - \mathbb{E}[Q]| > a ] \leq 2 \exp(-\frac{1}{3} a^2 l \frac{1}{\mathbb{E}[Q]} )
\end{equation}

Combining these two by setting $a = \epsilon(\mathbb{E}[Q] + \frac{k e^{-c}}{n})$,

\begin{equation}
\Pr[|\bar{Q} - \mathbb{E}[Q]| > \epsilon (\mathbb{E}[Q] + \frac{k e^{-c}}{n} )] \leq
\end{equation}
\begin{equation}
\leq 2 \exp(-\frac{1}{3} \epsilon^2 l (\mathbb{E}[Q] + \frac{k e^{-c}}{n} )^2 \frac{1}{\mathbb{E}[Q]} )
\end{equation}
\begin{equation}
 \leq 2 \exp(-\frac{2}{3} \epsilon^2 \frac{kl e^{-c}}{n}  )
\end{equation}

Thus, as long as $kl \geq \frac{2}{3} \frac{1}{\epsilon^2} n e^c \ln(1/\delta)$, then with $1- \delta$ probability, $\hat{Z} \in (1 \pm \epsilon) Z$
\end{proof}

 returns an unbiased estimate $\hat{Z}$ and for $kl = \frac{2}{3} \frac{1}{\epsilon^2} n \ln(1/\delta)$, then with $1- \delta$ probability, $\hat{Z} \in (1 \pm \epsilon) Z$

\subsection{EXPECTATION ESTIMATE}

\begin{thmn}[\ref{thm:expectation}]
Algorithm \ref{alg:expectation} returns an  estimate $\hat{F}$ such that $|\hat{F} - F| \leq \epsilon C$ with probability $\delta$ if

$$l k^2 \geq \frac{8 n^2 e^{2c}}{\epsilon^2} \ln(4/ \delta)$$

and

$$kl \geq \frac{8}{3} \frac{1}{\epsilon^2} n e^c \ln ( 2 / \delta )$$

\end{thmn}
\begin{proof}

Recall

$$J = \sum_i e^{y_i} f_i$$
$$\hat{J} = \sum_{i \in S} e^{y_i} f_i + \frac{n-k}{l} \sum_{i \in T} e^{y_i} f_i$$

Thus, $F = J/Z$ and $\hat{F} = \hat{J}/\hat{Z}$. 

To show that 

$$|\hat{F} - F| = |\frac{\hat{J}}{\hat{Z}} - \frac{J}{Z}| \leq \epsilon C$$

with probability $1 - \delta$, we will show that

$$|\frac{\hat{J}}{\hat{Z}} - \frac{\hat{J}}{Z}| \leq \frac{\epsilon}{2} C$$ 
$$|\frac{\hat{J}}{Z} - \frac{J}{Z}| \leq \frac{\epsilon}{2} C$$

each with probability $1 - \delta/2$. These will be shown as two separate parts.

\subsubsection{Part One}

Because

$$kl \geq \frac{2}{3} \frac{4}{\epsilon^2} n e^c \ln ( 2 / \delta )$$

from Theorem \ref{thm:partition}, with probability $1 - \delta/2$ then $\frac{|\hat{Z} - Z|}{Z} \leq \epsilon$ .

$$|\frac{\hat{J}}{\hat{Z}} - \frac{\hat{J}}{Z}| = \frac{|\hat{J}|}{\hat{Z}} \frac{ |\hat{Z} - Z| }{Z}$$
$$ \leq \frac{|\hat{J}|}{\hat{Z}} \frac{\epsilon}{2}$$
$$ \leq \frac{\epsilon}{2} C$$

\subsubsection{Part Two}

For the second one is written as the following lemma

\begin{lemma}
$|\frac{\hat{J}}{Z} - \frac{J}{Z}| \leq \frac{\epsilon}{2} C$ with probability $1 - \delta/2$ for 

$$lk^2 \geq \frac{8n^2}{\epsilon^2} \ln(4/\delta)$$

\end{lemma}
\begin{proof}

Note that the smallest element in $S$ has ``probability'' $e^{y_i}/Z \leq 1/k$. Thus, for the largest element not in $S$, $e^{y_i}/Z \leq e^c/k$. Define $Q$ as the random variable of the value of sampling $i$ uniformly from $[1,n] \setminus S$ and returning

$$Q = \frac{k e^{y_i}}{e^c Z} \frac{f_i}{C}$$

Note that $Q \in [-1, 1]$ and that $\mathbb{E}[Q] = \frac{1}{n - k} \frac{k}{Z C e^c} \sum_{i \not\in S} e^{y_i} f_i$. The elements of $T$ are samples $\{y^{(j)}\}_i$ and $\{f^{(i)}\}_i$ and thus we can define

$$Q^{(j)} = \frac{k e^{y^{(j)}}}{e^c Z} \frac{f^{(j)}}{C}$$

$$|\frac{J}{Z} - \frac{\hat{J}}{Z}| = \frac{1}{Z} | \sum_{i \not\in S} e^{y_i} f_i + \frac{n-k}{l} \sum_{j=1}^l e^{y^{(j)}} f^{(j)}|$$
$$ = \frac{(n - k)Ce^c}{k} | \mathbb{E}[Q] - \frac{1}{l} \sum_j Q^{(i)}|$$

From Hoeffding's Inequality,

$$\Pr[ | \mathbb{E}[Q] - \frac{1}{l} \sum_j Q^{(i)}| > t] \leq 2 \exp(-\frac{l t^2}{2})$$

Thus,

$$\Pr[ |\frac{J}{Z} - \frac{\hat{J}}{Z}| > \frac{(n - k)C e^c}{k} t] \leq 2 \exp(-\frac{l t^2}{2})$$

Defining $t = \frac{k \epsilon}{2 (n-k) e^c}$ we get that $\frac{(n-k)Ce^c}{k} t = \frac{\epsilon}{2} C$, so

$$\Pr[ |\frac{J}{Z} - \frac{\hat{J}}{Z}| > \frac{\epsilon}{2} C] \leq 2 \exp(-\frac{l k^2 \epsilon^2}{8  (n-k)^2 e^{2c}})$$

$$\Pr[ |\frac{J}{Z} - \frac{\hat{J}}{Z}| > \frac{\epsilon}{2} C] \leq 2 \exp(-\frac{l k^2 \epsilon^2}{8  n^2 e^{2c}})$$

Thus, the conclusion of the Lemma follows for

$$l k^2 \geq \frac{8 n^2 e^{2c}}{\epsilon^2} \ln(4/ \delta)$$

\end{proof}

With this lemma, the conclusion of the theorem follows.
\end{proof}

\begin{thmn}[\ref{thm:lsh_top}]
There exists a MIPS technique that returns the approximate top $k$ elements in sublinear time.
\end{thmn}

\begin{proof}

For the data structure, we create a sequence of LSH instances that are tuned to values that are $c/2$ apart. Thus, if $\|\theta\| \leq M_1$ and $\|\phi(x)\| \leq M_2$, then $| \theta \cdot \phi(x) | \leq M_1 M_2$. And we create $n_{LSH} = \frac{4M_1 M_2}{c}$ instances.

Call the LSH instances $\{L_i\}_i$ and for the $i^{th}$ instance, set the lower tuned value to be $S_{i,2} = (c/2)(i-1) - M_1 M_2$ and the higher tuned value to be $S_{i,1} = (c/2)i - M_1 M_2$. Thus, $S_1 - S_2 = c/2$. Further, set the failure probability of each LSH instance to be $\delta' = \delta k n_{LSH}$ so that with high probability, each of the LSH instances will not fail to find each of the top $k$ values.

At query time, hash the query $\theta$ and let $B_i$ be the buckets of neighbors from $L_i$. From the LSH guarantee, there are a small constant number of elements in $B_i$ that are smaller than $S_{i,2}$ and with high probability, all elements larger than $S_{i,1}$ will be in $B_i$. Find the neighboring pair of LSH instances $L_i$ and $L_{i+1}$ where $|B_{i+1}| \leq k$ and $|B_i| \geq k$. Collect $k' = k - |B_{i+1}|$ elements $B' \subseteq B_{i} - B_{i+1}$ where the elements are larger than $S_{i,2}$ (all but a constant number will be larger than $S_{i,2}$). Then return the elements $S = B' \cup B_{i+1}$.

Note that any elements larger than $S_{i+1,1}$ will be in $S$ with high probability because they will be contained in $B_{i+1}$. So $\max_{x \not\in S} \theta \cdot \phi(x) \leq S_{i+1,1}$. Further, by construction, $\min_{x \in S} \theta \cdot \phi(x) \geq S_{i,2}$. Thus, the technique returns the approximate top $k$ elements with high probability with a gap of $S_{i+1,1} - S_{i,2}=2(c/2)=c$.

This technique will have a total runtime of 

$$O( k + (\log(k) + \log(1/\delta)) \log(n) n^\rho)$$

where $\rho < 1$. Thus, we have a sublinear approximate top $k$ element MIPS technique.
\end{proof}

%\section{FIRST LEVEL HEADINGS}

%\subsection{SECOND LEVEL HEADING}

%\subsubsection{Third Level Heading}

\end{document}